%% file: iclr2027_conference.tex
\documentclass{article} 
\usepackage{iclr2027_conference,times}

\input{math_commands.tex}

\usepackage{hyperref}
\usepackage{url}

\usepackage{amsmath}
\usepackage{amssymb}
\usepackage{mathtools}
\usepackage{amsthm}
\usepackage{algorithm, algorithmic}
\usepackage{comment}
\usepackage{enumitem}

\theoremstyle{plain}
\newtheorem{theorem}{Theorem}
\newtheorem{lemma}{Lemma}
\theoremstyle{remark}

\title{DGA-Muon: Decoupled Geometry-Aligned Adaptive Scaling for Muon}

\author{Wenpeng Zhang* \\
Independent Reasearcher \\
\texttt{zhangwenpeng0@gmail.com} \\
\And
Runsheng Yu* \\
Independent Reasearcher \\
\texttt{runshengyu@gmail.com}
}

\iclrfinalcopy 
\begin{document}

\maketitle
\lhead{Preprint}

\begin{abstract}
While NorMuon has achieved strong empirical performance in pretraining, its underlying adaptive mechanism remains largely heuristic and poorly understood. In this work, we provide the first systematic theoretical analysis of NorMuon's adaptivity, revealing that it primarily arises from orthogonalization-induced geometry rather than genuine optimization dynamics, serving to offset the resulting geometric non-uniformity. Under exact orthogonalization, the adaptive scaling factors degenerate into a single global scalar for square and wide matrices, while for tall matrices their variation results from unevenly distributed row energy after orthogonalization. Under approximate orthogonalization, the orthogonality residuals introduce additional variation into the scaling, giving rise to a counterintuitive Orthogonalization--Adaptivity Paradox: more accurate orthogonalization weakens adaptivity. We further show that NorMuon's rigid row-wise scaling is geometrically misaligned with the one-sided orthogonal structure of tall matrices by distorting column orthogonality. Motivated by these limitations, we propose two core design principles that a desirable adaptive mechanism for Muon should satisfy. First, adaptive scaling should be decoupled from orthogonalization, with the scaling factors computed directly from raw gradients. Second, adaptive scaling should be aligned with the shape-dependent orthogonal structure of the polar factor, using row-wise scaling for wide matrices and column-wise scaling for tall matrices. We prove that this geometry-aligned scaling preserves the orthogonal structure of the update. By incorporating several other techniques, we obtain Decoupled Geometry-Aligned Muon (DGA-Muon). We establish convergence guarantees for DGA-Muon and empirically validate both our theoretical characterization of NorMuon's scaling degeneration and the superiority of DGA-Muon.
\end{abstract}

\section{introduction}

The Muon (MomentUm Orthogonalized by Newton--Schulz) optimizer~\citep{jordan2024muon}, specifically designed for optimizing two-dimensional parameters in neural networks, first gained traction through its remarkable empirical success in the Modded-NanoGPT speedrun~\citep{modded_nanogpt_2024}, where it demonstrated substantially faster convergence than AdamW~\citep{loshchilov2019decoupled}. Unlike AdamW, which treats weight matrices as flattened one-dimensional vectors, Muon preserves their inherent two-dimensional structure and applies an orthogonalization operation to equalize their singular values. Following this initial success, subsequent studies~\citep{liu2025muon,shah2025practical} further validated Muon's advantages over AdamW in the pretraining of multi-billion-parameter language models. To date, Muon has been widely adopted for training frontier foundation models, including open-weight models such as Kimi~\citep{team2025kimi,team2026kimi}, GLM~\citep{zeng2025glm,zeng2026glm}, and DeepSeek~\citep{xu2026deepseek}.

The remarkable success of Muon has sparked a surge of research interest. One line of research has sought to understand the mechanisms underlying Muon's superior performance~\citep{nguyen2026spectral,wang2026muon,wang2026muoniclr} and to establish theoretical guarantees for its convergence~\citep{li2025note,kim2026convergence}. An even larger body of research is devoted to further enhancing the practical performance of Muon. One prominent branch of these efforts focuses on developing new methods that improve the accuracy or efficiency of the orthogonalization procedure~\citep{amsel2026polarexpress,GramNewtonSchulz,liu2026muon}. Another promising direction explores adaptive preconditioning for Muon to better account for the geometry of the loss landscape. Conceptually, this mirrors how adaptive optimizers such as AdaGrad~\citep{duchi2011adaptive} and Adam~\citep{kingma2015adam} improved upon vanilla SGD. More fundamentally, a theoretical study~\citep{su2025isotropic} suggests that the perfectly uniform spectral flattening induced by Muon may not be optimal, providing a principled foundation for incorporating adaptive preconditioning into Muon. While several initial attempts have been made, most of these methods either incur prohibitive computational or memory costs (e.g., requiring expensive matrix decompositions or maintaining full-matrix preconditioners)~\citep{xu2026fismo, zhang2026mousse, song2026decoupling} or degrade to a simplistic single global scaling factor~\citep{zhang2025adagrad}. Striking an effective balance between computational efficiency and adaptive expressiveness, NorMuon~\citep{li2026normuon} has emerged as the most widely adopted adaptive Muon variant. By employing a lightweight yet fine-grained neuron-wise (i.e., row-wise) scaling strategy, NorMuon consistently outperforms vanilla Muon across various settings. Notably, it achieves state-of-the-art performance on two leading speedruns: Modded-NanoGPT and NanoChat~\citep{nanochat}.

A distinctive characteristic of NorMuon is that its adaptive scaling factors are computed from the orthogonalized momentum, rather than directly from historical gradients as in conventional adaptive optimizers such as AdaGrad~\citep{duchi2011adaptive}, Adam~\citep{kingma2015adam}, and Shampoo~\citep{gupta2018shampoo}. This design tightly couples adaptivity with the orthogonalization procedure, constituting a fundamentally different paradigm for adaptive optimization. However, this design choice is largely heuristic, and the rationale for using orthogonalized momentum as the source of adaptivity lacks theoretical justification. Thus far, neither theoretical convergence guarantees for NorMuon nor a principled characterization of its adaptive behavior have been established, leaving this unconventional adaptive mechanism poorly understood.

In this work, we provide the first systematic theoretical analysis of NorMuon's adaptivity. To understand its nature, we analyze the variation of NorMuon's adaptive scaling factors, asking what information is encoded in their differences and what drives this variation. Surprisingly, we find that NorMuon's adaptivity does not originate from conventional optimization dynamics, but instead arises primarily from the geometry induced by orthogonalization. To see this clearly, recall that NorMuon's scaling factors are derived from the squared \(\ell_2\)-norms of the rows of the orthogonalized matrix. Although orthogonalization equalizes singular values, it does not necessarily imply uniform row-wise norms. In fact, under exact orthogonalization, this row-wise non-uniformity constitutes the primary source of variation in the scaling factors. However, such variation arises exclusively for tall matrices, where orthogonalization enforces only one-sided orthogonality, namely column orthonormality, but not row orthonormality. Consequently, while the total squared $\ell_2$-norm across all rows is fixed, the squared $\ell_2$-norm of each individual row remains unconstrained, allowing the row-wise scaling factors to vary (see Subsection~\ref{derive} for a full derivation). In contrast, square and wide matrices satisfy row orthonormality, which ensures that the squared \(\ell_2\)-norm of every row is exactly one. Consequently, the row-wise adaptive scaling factors become identical across rows, and the adaptive scaling completely degenerates into a single global factor, thereby nullifying the intended row-wise adaptivity (see Subsection~\ref{exact} for detailed analysis). In practical implementations, orthogonalization is typically approximate. The resulting orthogonality residuals introduce an additional source of variation into the scaling factors. Notably, this gives rise to a rather counterintuitive and undesirable phenomenon, which we term the Orthogonalization--Adaptivity Paradox: the more accurate the orthogonalization, the weaker the resulting adaptivity, which vanishes entirely under exact orthogonalization for square and wide matrices; conversely, less accurate orthogonalization permits greater variation in the scaling factors and thus stronger adaptivity (see Subsections~\ref{approximate} and \ref{derive} for details). Together, these observations reveal a fundamental limitation of NorMuon: its adaptive scaling primarily acts as a compensatory mechanism that offsets geometric non-uniformity induced by orthogonalization, rather than capturing genuine optimization-relevant information.

The preceding analysis of NorMuon's inherent limitation suggests that a desirable adaptive mechanism for Muon should preserve NorMuon's computational efficiency while deriving adaptivity from optimization-relevant signals rather than orthogonalization-induced geometry. At its core, NorMuon's limitation arises from computing scaling factors from the orthogonalized momentum, which inevitably causes the resulting adaptivity to inherit geometric distortions induced by orthogonalization. To break this coupling, we introduce a decoupled design principle: computing adaptive scaling factors directly from raw gradients, independent of orthogonalization, and subsequently applying them to the orthogonalized matrix, i.e., the polar factor~\citep{higham2008functions,higham2020polar} of the momentum. Another design deficiency of NorMuon lies in its reliance on rigid row-wise scaling, which overlooks a fundamental geometric reality: for non-square matrices, the polar factor inherently exhibits only one-sided orthogonality. Specifically, for tall matrices, the polar factor is column-orthogonal, with its columns defining the orthogonal directions. Row-wise scaling, however, is a one-sided operation applied exclusively along the row dimension, which is misaligned with the column-wise orthogonal structure. As a result, this geometric mismatch distorts the column geometry and breaks the original orthogonality induced by the polar factor. To prevent this structural distortion, we introduce our second design principle: geometry-aligned scaling, which aligns the scaling with the one-sided orthogonal structure. Specifically, the scaling axis is chosen according to the shape-dependent orthogonality structure of the polar factor: row-wise scaling is applied for wide matrices, whereas column-wise scaling is applied for tall matrices (or either for square matrices), thereby preserving the orthogonal structure obtained through orthogonalization. We further provide a formal theoretical
justification for the orthogonality-preserving property of the proposed adaptive scaling mechanism (see Subsection~\ref{formal} for the proof).

Beyond the two core design principles, we incorporate several additional design considerations to better regulate the dynamics of the adaptive scaling factors. First, we construct the second-moment estimate for each adaptive unit (i.e., row or column) from the sum, rather than the mean, of the squared gradient entries, thereby making the scaling of each unit depend on its total accumulated gradient energy. This differs from NorMuon, which normalizes the total row-wise gradient energy by the row dimension, making its adaptive scaling depend on average per-coordinate energy. Second, we apply standard bias correction to the second-moment estimates to mitigate the initialization bias of the exponential moving average (EMA) during the early stages of optimization. Third, to curb the excessive anisotropy that adaptive scaling can introduce across different orthogonal directions, we employ an adaptive clipping mechanism that bounds the dynamic range of the scaling factors. This preserves Muon's favorable ability to support large step sizes while maintaining training stability. Combining all these ingredients, we propose Decoupled Geometry-Aligned Muon (DGA-Muon; see Algorithm~\ref{alg:A-Muon}), which equips Muon with geometry-aligned adaptive scaling based on genuine gradient signals while achieving computational efficiency comparable to that of NorMuon.

We establish convergence guarantees for DGA-Muon under standard assumptions in the non-convex setting. We empirically verify our theoretical characterization of NorMuon, showing that removing its adaptive scaling on the square Q/K/V/O projection matrices in NanoChat leaves the validation performance essentially unchanged. We further demonstrate that DGA-Muon substantially outperforms NorMuon in an end-to-end NanoChat evaluation, validating the effectiveness of our proposed adaptive scaling mechanism. Finally, we note that our analysis and algorithmic design provide broader insights into general matrix-aware optimizer design beyond enhancing Muon or NorMuon. Our proposed principles extend naturally to directions other than those used by Muon. More generally, any directions that meaningfully characterize the underlying matrix can serve as candidates for adaptive scaling. Furthermore, our analysis highlights the intrinsic necessity of row-wise/column-wise adaptivity as a natural mechanism for adapting to the typical one-sided directional structure inherent in matrices, which manifests along either the row or column dimension.

\section{Preliminaries: Muon and NorMuon}

In this section, we briefly review Muon and NorMuon, providing the basis for our subsequent theoretical analysis.

\paragraph{Muon.}  Muon is an optimizer tailored for two-dimensional matrix parameters. It maintains a momentum buffer $M_t$, updated according to
$M_t = \beta_1 M_{t-1} + (1 - \beta_1) G_t$,
where $G_t$ denotes the gradient at iteration $t$, $\beta_1 \in [0,1)$ is the momentum coefficient, and $M_0 = \mathbf{0} \in \mathbb{R}^{m \times n}$. However, rather than updating the parameters directly along the raw momentum direction, Muon applies the update along the direction of the orthogonalized momentum. Specifically, given the current weight matrix $W_{t-1} \in \mathbb{R}^{m \times n}$ and the momentum $M_t \in \mathbb{R}^{m \times n}$, the update rule of Muon is given as follows
\begin{equation*}
    W_t = W_{t-1} - \eta\operatorname{Orth}(M_t)
\end{equation*}
where $\eta$ is the learning rate, and $\operatorname{Orth}(M_t)$ denotes the orthogonalization of $M_t$, i.e., $\operatorname{Orth}(M_t) = U V^\top$ with $M_t = U \Sigma V^\top$ being the (thin) singular value decomposition\,(SVD). The orthogonalization procedure discards all singular value information by setting every singular value to one, while preserving the directional structure induced by the singular vectors. The resulting orthogonalized matrix \(UV^\top\) is commonly referred to as the \textit{polar factor}~\citep{higham2008functions,higham2020polar} of \(M_t\). In practice, however, Muon avoids the high cost of computing the full SVD by approximating the orthogonalization using efficient Newton--Schulz iterations~\citep{jordan2024muon, amsel2026polarexpress}.

\paragraph{NorMuon.} NorMuon extends Muon by introducing row-wise adaptive scaling. In addition to the momentum buffer $M_t \in \mathbb{R}^{m \times n}$, it maintains a second-moment estimate $v_t \in \mathbb{R}^m$. Specifically, $v_t$ tracks the row-wise mean square of the orthogonalized momentum. Rather than directly updating the parameters with the orthogonalized momentum, NorMuon first scales it by the inverse square root of $v_t$ before applying the update. Formally, let $O_t=\operatorname{Orth}(M_t)$ and $\operatorname{mean}_{\mathrm{cols}}(\cdot)$ denote the row-wise mean, i.e., the mean across columns for each row. Given the decay rate $\beta_2 \in [0,1)$ and the initialization $v_0=\mathbf{0} \in \mathbb{R}^m$, the second-moment estimate $v_t$ is computed via an exponential moving average (EMA)
\begin{equation*}
    v_t = \beta_2 v_{t-1} + (1 - \beta_2) \operatorname{mean}_{\mathrm{cols}}(O_t \odot O_t),
\end{equation*}
where $\odot$ denotes the Hadamard (element-wise) product. The final parameter update then uses $O_t$ scaled by the inverse square root of $v_t$
\begin{equation*}
    W_t = W_{t-1} - \eta_t \frac{O_t}{\sqrt{v_t} + \delta}.
\end{equation*}
Here, the division is broadcast row-wise, $\eta_t$ is the learning rate, and $\delta > 0$ is a small constant for numerical stability. For clarity, we omit the weight decay and learning rate scaling used in the original NorMuon, as neither affects the subsequent analysis.

\section{On the Origin and Limitations of NorMuon's Adaptivity}
\label{limitation}

Although NorMuon's row-wise adaptive scaling has demonstrated strong empirical performance, its underlying theoretical foundation remains largely unexplored.\footnote{To the best of our knowledge, no theoretical convergence guarantee has yet been established for NorMuon.} In this section, we theoretically analyze the origin of NorMuon's adaptivity and uncover its inherent limitations. Specifically, we show that this adaptivity primarily arises from orthogonalization-induced geometry rather than optimization dynamics, and further demonstrate that such adaptivity can degenerate under certain conditions.

\subsection{Exact Orthogonalization}
\label{exact}
 
We first consider the $m \le n$ regime, where the matrix is square ($m = n$) or wide ($m < n$). To facilitate the subsequent analysis, we express the row-wise mean of the element-wise square \(O_t \odot O_t\) in terms of the Gram matrix of $O_t$ 
\begin{equation*}
    \operatorname{mean}_{\mathrm{cols}}(O_t \odot O_t) = \frac{1}{n} \operatorname{diag}(O_t O_t^\top),
\end{equation*}
where $\operatorname{diag}(\cdot)$ denotes the operator that extracts the diagonal entries of a square matrix into a vector. Recall that the thin SVD of $M_t$ is given by $M_t = U\Sigma V^\top$. For $m \le n$, assuming that $M_t$ has full row rank, we have $U \in \mathbb{R}^{m \times m}$, $\Sigma \in \mathbb{R}^{m \times m}$, and $V \in \mathbb{R}^{n \times m}$. By definition of the SVD, $U$ is orthogonal ($UU^\top = U^\top U = I_m$) and $V$ is column-orthonormal ($V^\top V = I_m$). Then, under exact orthogonalization, the identity $O_t = UV^\top$ holds exactly. Therefore, we obtain
\begin{equation*}
    O_t O_t^\top = (UV^\top)(UV^\top)^\top = U(V^\top V)U^\top = UU^\top = I_m,
\end{equation*}
i.e., $O_t$ is row-orthonormal, and hence
\begin{equation*}
    \operatorname{mean}_{\mathrm{cols}}(O_t \odot O_t) = \frac{1}{n} \operatorname{diag}(O_t O_t^\top) = \frac{1}{n} \mathbf{1}_m,
\end{equation*}
where $\mathbf{1}_m$ denotes the all-ones vector in $\mathbb{R}^m$. 
Substituting this identity into the EMA update rule for $v_t$, unrolling the recursion, evaluating the resulting geometric series, and using $v_0 = \mathbf{0}$, we obtain
\begin{equation*}
    v_t = \beta_2^t v_0 + (1-\beta_2^t)\frac{1}{n}\mathbf{1}_m
    = (1-\beta_2^t)\frac{1}{n}\mathbf{1}_m,
\end{equation*}
which is uniform across rows, with only a global scalar varying over time. Furthermore, if the standard bias correction omitted in NorMuon were applied, it would cancel the time-dependent factor, yielding exactly $ \tfrac{1}{n} \mathbf{1}_m$, which is functionally equivalent to a fixed global scalar throughout training. The resulting scaling vector $\tfrac{1}{\sqrt{v_t}+\delta}$ therefore assigns the same scaling coefficient to each row. Hence, the originally intended neuron-wise adaptive scaling completely collapses to a trivial global scalar, thereby rendering the adaptive mechanism fundamentally ineffective.

\subsection{Approximate Orthogonalization}
\label{approximate}

The preceding analysis assumes exact orthogonalization. In practice, however, NorMuon employs approximate orthogonalization via a finite number of Newton--Schulz iterations~\citep{li2026normuon}. We now analyze the approximate orthogonalization case. Assume that the approximately orthogonalized matrix $O_t$ satisfies
\[
O_tO_t^\top = I_m + E_t,
\]
where $E_t \in \mathbb{R}^{m \times m}$ denotes the orthogonality residual induced by approximate orthogonalization. Then we have
\[
\operatorname{mean}_{\mathrm{cols}}(O_t\odot O_t)
=
\frac1n\operatorname{diag}(O_tO_t^\top)
=
\frac1n\left(\mathbf1_m+\operatorname{diag}(E_t)\right).
\]
Substituting this expression into the EMA update of $v_t$, unrolling the recursion, evaluating the resulting series, and using $v_0=\mathbf0$, yields
\[
v_t = (1-\beta_2^t)\frac1n\mathbf1_m + \frac{1-\beta_2}{n} \sum_{i=1}^{t} \beta_2^{t-i} \operatorname{diag}(E_i).
\]
For convenience, define
\[
R_t = \frac{1-\beta_2}{n} \sum_{i=1}^{t} \beta_2^{t-i} \operatorname{diag}(E_i).
\]
The term $(1-\beta_2^t)\frac1n\mathbf1_m$ coincides exactly with the row-invariant scaling term derived in the exact orthogonalization case, whereas $R_t$ captures the exponentially weighted accumulation of the orthogonality residuals and constitutes the only source of row-wise variation in $v_t$. To further characterize the magnitude of $R_t$, we assume that the orthogonality residual is uniformly bounded in spectral norm by some small constant $\epsilon > 0$, i.e., $\|E_t\|_2 \le \epsilon$ for all $t$. This assumption is consistent with practical Newton--Schulz implementations, where the orthogonalization error is controlled within a prescribed tolerance. The spectral norm bound directly implies a bound on the diagonal entries of \(E_t\). Specifically, for each $j \in \{1, \dots, m\}$, by the Cauchy--Schwarz inequality, the definition of the induced operator norm, and the fact that \(\|\mathbf{e}_j\|_2=1\), we have
\[
|(E_t)_{jj}| = |\mathbf{e}_j^\top E_t \mathbf{e}_j| \leq \|\mathbf{e}_j\|_2\|E_t \mathbf{e}_j\|_2 \leq \|E_t\|_2 \leq \epsilon.
\]
Consequently, we can bound \(R_t\) in the infinity norm as
\[
\|R_t\|_\infty \le \frac{1-\beta_2}{n} \sum_{i=1}^{t} \beta_2^{t-i} \,\epsilon = (1-\beta_2^t)\frac{\epsilon}{n} \le \frac{\epsilon}{n}.
\]
It immediately follows that the difference between any two entries of \(v_t\) is entirely determined by \(R_t\). That is, for any \(j,k \in \{1, \dots, m\}\), we have
\[
|(v_t)_j-(v_t)_k|
=
|(R_t)_j-(R_t)_k|
\le
2\|R_t\|_\infty
\le
\frac{2\epsilon}{n}.
\]
This bounded variation in \(v_t\) directly translates into a bounded variation in the final adaptive scaling factors \(\frac{1}{\sqrt{v_t}+\delta}\). Specifically, by the Mean Value Theorem, there exists some constant \(C>0\) such that, for any \(j,k \in \{1, \dots, m\}\), we have
\[ \left| \frac{1}{\sqrt{(v_t)_j}+\delta} - \frac{1}{\sqrt{(v_t)_k}+\delta} \right| \leq \frac{C\epsilon}{n}. \]
In practice, Newton--Schulz-based orthogonalization rapidly reduces the orthogonality residual as iterations proceed, typically resulting in a very small \(\epsilon\). Consequently, the adaptive scaling factors exhibit only negligible variation across rows, rendering the row-wise adaptivity practically ineffective.

In short, the above analysis reveals that, in the \(m\le n\) regime, the apparent row-wise adaptivity in NorMuon is not driven by the underlying gradient statistics or curvature information that adaptive optimizers are typically designed to capture. Instead, the remaining row-wise variation in the adaptive scaling factors is solely induced by the orthogonalization residual. In other words, the observed adaptivity of NorMuon is merely a byproduct of inexact orthogonalization. As \(\epsilon\to0\), this residual-induced adaptivity vanishes, and the behavior of approximate orthogonalization converges to that of exact orthogonalization, recovering uniform, non-adaptive row-wise scaling.

\paragraph*{Remark.} A counterintuitive implication of the preceding analysis, which we term the \textbf{\textit{Orthogonalization--Adaptivity Paradox}}, is that the adaptivity of NorMuon is inversely coupled with the quality of orthogonalization. The more accurate the orthogonalization, the weaker the adaptivity becomes; in the limit of exact orthogonalization, the adaptivity disappears entirely. Conversely, less accurate orthogonalization permits greater variation in the row-wise scaling factors and thus allows for stronger adaptivity. This suggests that NorMuon's adaptive mechanism essentially serves to compensate for imperfections in orthogonalization, rather than providing genuine optimization-driven adaptivity. This stands in stark contrast to the design philosophy of Muon, where accurate orthogonalization is the cornerstone of the optimizer. Consequently, NorMuon's adaptive mechanism is intrinsically incompatible with ongoing research efforts to improve Muon through increasingly accurate orthogonalization~\citep{grishina2025accelerating, he2025root, amsel2026polarexpress}.

Note that the above analysis relies on the full row rank assumption. In fact, this already covers the generic case encountered in practice. Rank deficiency is an undesirable edge case that can hinder effective training and impair the model's representational capacity, and is thus generally avoided in practice. Since it rarely occurs, we do not treat it separately here or in the subsequent analysis. We do emphasize, however, that even under rank deficiency, the row-wise scaling of NorMuon remains entirely determined by the orthogonalization procedure rather than by the optimization dynamics. A more detailed discussion of this case will be provided in an extended version of this work.

\subsection{The \texorpdfstring{$m > n$}{m > n} Regime}
\label{derive}

In the $m > n$ regime, where the matrix is tall, the situation is fundamentally different. The row-orthonormality property $O_t O_t^\top = I_m$ no longer holds due to the dimensional constraint $m > n$. Instead, under exact orthogonalization, the column-orthonormality property $O_t^\top O_t = I_n$ holds. This implies that every column of $O_t$ has unit norm, i.e., $\|(O_t)_{:,j}\|_2^2 = 1$ for each $j \in \{1, \dots, n\}$. In contrast, since $O_t O_t^\top \neq I_m$, 
the squared norm of each row of $O_t$ is no longer constrained to $1$. Nevertheless, the aggregate row energy is constrained to be fixed. Specifically, we have
\begin{equation*}
\sum_{i=1}^m \|(O_t)_{i,:}\|_2^2 = \sum_{i=1}^m \sum_{j=1}^n (O_t)_{ij}^2 = \sum_{j=1}^n \sum_{i=1}^m (O_t)_{ij}^2 = \sum_{j=1}^n \|(O_t)_{:,j}\|_2^2 = n,
\end{equation*}
where the last equality follows from the column-orthonormality of $O_t$. Since there is no row orthonormality constraint, this fixed aggregate energy is no longer forced to be uniformly distributed across the $m$ rows. Consequently, the squared norm of each row can deviate substantially from the uniform allocation $n/m$, inducing row-wise variation in $\operatorname{mean}_{\mathrm{cols}}(O_t \odot O_t)$ and hence in the resulting adaptive scaling factors. This indicates that, in the $m > n$ regime, NorMuon's adaptivity is not driven by conventional optimization dynamics underlying most existing adaptive optimizers, but instead stems from the intrinsic uneven row-energy distribution induced by tall-matrix orthogonalization.

The above analysis focuses on exact orthogonalization. For approximate orthogonalization, the analysis largely follows that of the $m \le n$ regime, where deviations from exact orthogonality further introduce residual-induced variation in the adaptive scaling factors. This residual effect also diminishes as the orthogonalization accuracy improves, following the same Orthogonalization--Adaptivity Paradox.

\paragraph*{Summary.} Our analysis in the three preceding subsections reveals that NorMuon's adaptivity is not optimization-driven in the conventional sense. Instead, it primarily serves as a compensation mechanism for the non-uniformity induced by orthogonalization, including residual errors and uneven row-energy distribution. Although this compensation can provide empirical benefits, the resulting adaptivity is largely an orthogonalization-induced \textit{byproduct} rather than a principled adaptive mechanism. Furthermore, approximate orthogonalization gives rise to the Orthogonalization--Adaptivity Paradox, revealing an inherent limitation of this form of adaptation. These findings motivate the design of a more principled adaptive mechanism tailored specifically to Muon.

\section{Decoupled Geometry-Aligned Adaptive Scaling}

In this section, we present a novel adaptive scaling mechanism specifically designed for Muon. We first outline the core design principles and then formally justify the orthogonality-preserving property of the resulting scaling strategy. We further discuss additional design considerations for constructing the adaptive scaling factors. The complete optimization procedure is summarized in Algorithm~\ref{alg:A-Muon}\footnote{We omit the weight decay term here for generality, as it is independent of the core algorithm and various variants can be seamlessly incorporated (e.g., standard or cautious weight decay~\citep{chen2026cautious}).}.

\begin{algorithm}[t]
\caption{Decoupled Geometry-Aligned Adaptive Muon (\textbf{DGA-Muon})}
\label{alg:A-Muon}
\begin{algorithmic}[1]
\STATE \textbf{Input:} Step size $\eta > 0$, decay rates $\beta_1, \beta_2 \in [0, 1)$, stabilizer $\delta > 0$, clip ratios $c \in (0, 1], d \ge 1$\STATE \textbf{Initialize:} $W_1 \in \mathbb{R}^{m \times n}$, $M_0 = \mathbf{0} \in \mathbb{R}^{m \times n}$, $v_0 = \mathbf{0} \in \mathbb{R}^{\min(m, n)}$
\FOR{$t=1,\dots,T$}
    \STATE Compute stochastic gradient $G_t$
    \STATE Update first moment estimate
    \(M_t = \beta_1 M_{t-1} + (1-\beta_1)G_t\)
    \STATE Update second moment estimate from raw gradients\vspace{-1mm}
    \[
    v_t = 
    \begin{cases}
    \beta_2 v_{t-1} + (1-\beta_2)\operatorname{diag}(G_t G_t^\top), & m \le n, \\
    \beta_2 v_{t-1} + (1-\beta_2)\operatorname{diag}(G_t^\top G_t), & m > n.
    \end{cases}
    \]\vspace{-1mm}
    \STATE Apply bias correction $\hat{v}_t = v_t / (1 - \beta_2^t)$
    \STATE Compute adaptive scaling factors $s_t = 1 / (\sqrt{\hat{v}_t} + \delta)$
    \STATE Compute median of scaling factors $s_{\text{med}} = \operatorname{median}(s_t)$
    \STATE Clip scaling factors $s_t \leftarrow \operatorname{clip}(s_t, \, c \cdot s_{\text{med}}, \, d \cdot s_{\text{med}})$
    \STATE Orthogonalize first moment $O_t = \operatorname{Orth}(M_t)$
    \STATE Apply geometry-aligned adaptive scaling\vspace{-1mm}
    \[
    \widetilde{O}_t = 
    \begin{cases}
    \operatorname{Diag}(s_t)\hspace{0.3mm} O_t, & m \le n, \\
    O_t\operatorname{Diag}(s_t), & m > n.
    \end{cases}\qquad\hspace{5mm}
    \]\vspace{-1mm}
    \STATE Update parameters $W_{t+1} = W_t - \eta\hspace{0.2mm} \widetilde{O}_t$
\ENDFOR
\end{algorithmic}
\end{algorithm}

\subsection{Design Principles: Decoupled Adaptivity and Orthogonality Preservation}

We begin by examining the root cause of the limitations of NorMuon's adaptivity identified in the preceding section. Specifically, our analysis reveals that these limitations stem from the inherent coupling between orthogonalization and adaptive scaling. Because the scaling factors are computed from the orthogonalized momentum, the resulting adaptivity is inherently dictated by orthogonalization-induced geometry rather than direct gradient information. This analysis naturally motivates a core design principle: the computation of adaptive scaling factors should be decoupled from the orthogonalization procedure. Guided by this principle, we propose a decoupled adaptive mechanism that computes the adaptive scaling factors directly from the raw gradients, rather than from the orthogonalized momentum, while applying these factors to the orthogonalized update, i.e., the polar factor. In this way, the adaptive scaling is no longer dictated by orthogonalization-induced geometry, but instead reflects genuine optimization-relevant information---a defining characteristic of standard adaptive optimizers.

While decoupling the computation of adaptive scaling factors resolves the fundamental coupling between adaptivity and orthogonalization, a closer examination of the geometric structure of the orthogonalized update reveals another design deficiency in NorMuon. Specifically, in the $m > n$ regime, the polar factor has orthonormal columns, which define the orthogonal directions of the update, whereas NorMuon applies adaptive scaling along the row dimension. This dimensional mismatch inevitably distorts these orthogonal directions and destroys the orthogonality that the polar factor is designed to preserve. This observation motivates a second design principle: adaptive scaling should respect the intrinsic geometry of the orthogonalized update. Intuitively, since orthogonalization is designed to extract a set of orthogonal update directions, any subsequent adaptive scaling should preserve this directional structure rather than distort it. Following this principle, we introduce a geometry-aligned scaling mechanism that applies row-wise scaling for $m < n$ and column-wise scaling for $m > n$, while either choice can be used when $m = n$, thereby preserving the orthogonality of the update in all regimes.

\subsection{Formal Justification of the Orthogonality-Preserving Property}
\label{formal}

We now provide a formal justification of the orthogonality-preserving property for the proposed geometry-aligned scaling mechanism. We first consider the $m>n$ regime under exact orthogonalization, where the polar factor $O_t\in\mathbb{R}^{m\times n}$ satisfies $O_t^\top O_t=I_n$. In this setting, the columns of $O_t$ form a set of mutually orthogonal directions. Let $v_t\in\mathbb{R}^n$ denote the column-wise second-moment estimate, where each entry tracks the EMA of the squared $\ell_2$-norm of the corresponding column of the raw gradients, i.e.,
\begin{equation*}
v_t = \beta_2 v_{t-1} + (1-\beta_2)\operatorname{diag}\!\left(G_t^\top G_t\right).
\end{equation*}
Note that the Gram matrix here is introduced only for notational convenience; in practice, only its diagonal entries are needed, so the full Gram matrix is never explicitly formed. The associated diagonal scaling matrix is constructed as $H_t = \operatorname{Diag}(1/(\sqrt{v_t}+\delta))  \in\mathbb{R}^{n\times n}$, where $\operatorname{Diag}(\cdot)$ denotes the operator that constructs a diagonal matrix from a vector. Applying column-wise adaptive scaling to  $O_t$ is equivalent to right-multiplying \(O_t\) by $H_t$, yielding the scaled update $\widetilde{O}_t=O_tH_t$. The Gram matrix of the scaled update satisfies
$$\widetilde{O}_t^\top \widetilde{O}_t = H_t^\top O_t^\top O_t H_t = H_t^\top \!I_n H_t =H_t^2,$$ 
which remains diagonal. This implies that the columns of $\widetilde{O}_t$ remain mutually orthogonal, i.e.,
\[\langle \widetilde{O}_{:,i},\, \widetilde{O}_{:,j} \rangle = 0,\,\, i \neq j,
\]
while their squared norms satisfy
\[
\|\widetilde{O}_{:,i}\|_2^2 = (H_t^2)_{ii},
\]
indicating that the magnitude of each orthogonal direction is independently scaled by its corresponding adaptive scaling factor. Thus, column-wise adaptive scaling transforms the polar factor with unit column norms into an update with orthogonal columns but heterogeneous column norms. It preserves the orthogonal directions extracted by orthogonalization while assigning an adaptive step size to each direction. In contrast, applying row-wise adaptive scaling to $O_t$ is equivalent to left-multiplying $O_t$ by the corresponding diagonal scaling matrix $D_t\in\mathbb{R}^{m\times m}$, constructed analogously from the row-wise second-moment estimate, yielding the scaled update
$\widetilde{O}_t = D_t O_t$. The resulting Gram matrix becomes
\[
\widetilde{O}_t^\top \widetilde{O}_t = O_t^\top D_t^\top D_t O_t = O_t^\top D_t^2 O_t,
\]
which is generally non-diagonal. This indicates that row-wise adaptive scaling reintroduces coupling between the orthogonal column directions extracted by orthogonalization, and hence the orthogonal structure established by orthogonalization is no longer preserved in the $m>n$ regime.

For the $m<n$ regime, the analysis proceeds analogously, with the roles of rows and columns reversed. In this case, we construct the adaptive scaling factors using the row-wise second-moment estimate $v_t\in\mathbb{R}^m$, where each entry tracks the EMA of the squared $\ell_2$-norm of the corresponding row of the raw gradients, i.e., 
\begin{equation*}
v_t
=
\beta_2 v_{t-1}
+
(1-\beta_2)\operatorname{diag}(G_tG_t^\top).
\end{equation*}
The polar factor $O_t$ satisfies $O_t O_t^\top = I_m$, implying that the rows of $O_t$ form a set of mutually orthogonal directions. Applying row-wise adaptive scaling to $O_t$ is equivalent to left-multiplying $O_t$ by the corresponding diagonal scaling matrix, thereby preserving the row-wise orthogonal structure. In contrast, applying column-wise adaptive scaling, equivalently right-multiplying $O_t$ by the corresponding diagonal scaling matrix, generally distorts the row-wise orthogonal structure. For the \(m=n\) regime, the polar factor \(O_t\) is a square orthogonal matrix, satisfying both \(O_t O_t^\top = I_m\) and \(O_t^\top O_t = I_m\). Hence, both the rows and columns of \(O_t\) form sets of mutually orthogonal directions, and either row-wise or column-wise adaptive scaling preserves the corresponding orthogonal directions, adjusting only their magnitudes.

\noindent\textbf{Remark: Broader Insights for Matrix-Aware Optimizer Design.} 
The design principles and geometric justifications presented above provide deeper insights into the general paradigm of matrix-aware optimization.
\begin{enumerate}[leftmargin=15pt]
\item \textbf{A Generalized Design Philosophy.} Our decoupled, direction-aware adaptive mechanism transcends being a mere enhancement for Muon; it constitutes a more general design philosophy for matrix-aware optimizers. Notably, the directions utilized need not be restricted to the orthogonal directions used by Muon, but can naturally extend to broader classes of directions that characterize matrices.
\item \textbf{Intrinsic Necessity of Row-Wise/Column-Wise Adaptivity.} This perspective formally establishes why row-wise or column-wise adaptivity is fundamentally necessary: it provides a natural mechanism for adapting to the one-sided orthogonal directions of matrices. Rather than viewing row-wise/column-wise scaling as a structural approximation to Adam-like element-wise scaling (as in Adam-mini~\citep{zhang2025adam} and other related work), our approach grounds it in the intrinsic geometry of matrices~\citep{zhang2026matrixadagrad}. Furthermore, although methods such as NorMuon also employ row-wise/column-wise adaptation, our motivation and design arise from a distinct geometric alignment principle rather than heuristic normalization.
\end{enumerate}

\subsection{Additional Design Considerations}

Beyond the two aforementioned design principles, we further incorporate several additional considerations when constructing the adaptive scaling factors to better regulate their dynamics during optimization. First, the second-moment estimate is constructed from the sum rather than the mean of the squared entries. While this is already implicit in the construction of $v_t$ above, we make this design choice more explicit here. For a row or column gradient vector, the corresponding second-moment term is given by its squared $\ell_2$-norm, without normalization by the vector dimension. This reflects our design philosophy of treating each row (or column) gradient vector as a single adaptive unit, with its scaling determined by the total accumulated gradient energy of the vector. Consequently, units with larger accumulated gradient energy receive smaller scaling factors, while units with lower accumulated energy are assigned larger ones. This contrasts with NorMuon, which normalizes the row-wise gradient energy by the row dimension, making its adaptive scaling depend on the average per-coordinate energy rather than the total energy of each row.

Since the second-moment estimate $v_t$ is computed directly from raw gradients and is decoupled from orthogonalization, it suffers from the standard initialization bias of EMA estimators. In particular, during the early optimization steps, $v_t$ tends to underestimate the underlying second moment due to the zero initialization of the accumulator. We therefore apply the standard bias-correction technique to obtain a bias-corrected estimate \(\hat{v}_t = \frac{v_t}{1-\beta_2^t},\) which is then used to compute the adaptive scaling factor \(s_t=\frac{1}{\sqrt{\hat{v}_t}+\delta}\).\footnote{Bias correction can be replaced with a gradient normalization technique, which we will include in revision.} Here \(v_t\), \(\hat{v}_t\), and \(s_t\) are generic notations that can represent vectors in either \(\mathbb{R}^{m}\) or \(\mathbb{R}^{n}\). Note that the first moment $M_t$ requires no bias correction, owing to the scale-invariance of the subsequent orthogonalization operation~\citep{pethick2025training}.

While introducing adaptive scaling into Muon provides more fine-grained control over different adaptive units, it may also lead to excessive anisotropic scaling, resulting in overly imbalanced updates across different orthogonal directions. Such excessive anisotropy necessitates the use of smaller step sizes to maintain stable optimization. To preserve Muon’s favorable property of supporting large step sizes, we introduce a clipping operation that bounds the degree of adaptivity. Specifically, we apply the following element-wise clipping operation
\[
\operatorname{clip}(s_t,s_{t,\min},s_{t,\max})
=
\min(\max(s_t,s_{t,\min}),\,s_{t,\max}).
\]
We determine the clipping bounds $s_{t,\min}$ and $s_{t,\max}$ adaptively based on the median of the scaling factors, denoted by \(s_{\mathrm{med}}=\operatorname{median}(s_t)\). Specifically, we set $s_{t,\min} = c \cdot s_{\mathrm{med}}$ and $s_{t,\max} = d \cdot s_{\mathrm{med}},$
where $c \in (0, 1]$ and $d \ge 1$ control the allowable dynamic range. One can also use a single hyperparameter by setting \(d=1/c\). The clipping operation effectively bounds the condition number of the induced diagonal scaling matrix and prevents disproportionately large or small updates along specific orthogonal directions. When \(s_{t,\min}=s_{t,\max}\), the clipping removes all relative adaptivity, reducing the update to a uniformly scaled Muon update.

\section{Theoretical Analysis}

In this section, we establish the convergence guarantee of DGA-Muon under exact orthogonalization. We first introduce the assumptions required for the analysis.

\paragraph{Assumptions.}
We make the following assumptions throughout the analysis.
\begin{enumerate}
    \item[(A1)] (\(L\)-smoothness)
    \(\|\nabla f(X)-\nabla f(Y)\|_F
    \le L\|X-Y\|_F,
    \)
    for all \(X,Y\).
\item[(A2)] (Unbiased stochastic gradients with bounded variance)
    The stochastic gradients satisfy
    \[
    \mathbb{E}[G_t\mid\mathcal{F}_{t-1}]
    =
    \nabla f(W_t),
    \quad
    \mathbb{E}\|G_t-\nabla f(W_t)\|_F^2
    \le \sigma^2 .
    \]
\item[(A3)] (Bounded stochastic gradients)
There exists a constant $G_{\max}>0$ such that
\[
\|G_t\|_F \le G_{\max} \,\,\, \text{for all } t.
\]
\end{enumerate}

We now proceed with the convergence analysis. By the $L$-smoothness of $f$ and the update rule $W_{t+1} = W_t-\eta \hspace{0.3mm}\widetilde{O}_t$, we have
\[
\begin{aligned}
f(W_{t+1})
&\le f(W_t)
    + \langle \nabla f(W_t),\, W_{t+1}-W_t \rangle_F
    + \frac{L}{2}\|W_{t+1}-W_t\|_F^2 \\
&= f(W_t)
    - \eta \langle \nabla f(W_t),\, \widetilde{O}_t \rangle_F
    + \frac{\eta^2L}{2}\|\widetilde{O}_t\|_F^2 .
\end{aligned}
\]
We first analyze the quadratic term $\|\widetilde{O}_t\|_F^2$. For the \(m\le n\) regime, the geometry-aligned adaptive update \(\widetilde{O}_t\) is given by \(\widetilde{O}_t=\operatorname{Diag}(s_t)\hspace{0.2mm}O_t\). Since \(O_t=U_tV_t^\top\) is the polar factor of \(M_t\), we have
\(O_tO_t^\top
= U_tV_t^\top V_tU_t^\top = U_tU_t^\top
= I_m .\)
Then using the definition of the Frobenius norm and the cyclic property of the trace, we have
\[
\|\widetilde{O}_t\|_F^2 = \operatorname{tr}(\widetilde{O}_t^\top \widetilde{O}_t) = \operatorname{tr}\left(O_t^\top\operatorname{Diag}(s_t)^2O_t\right) = \operatorname{tr}\left(\operatorname{Diag}(s_t)^2O_tO_t^\top\right) = \operatorname{tr}\left(\operatorname{Diag}(s_t)^2\right) = \sum_{i=1}^{m}s_{t,i}^{2}.
\]
Let $s_{t,\max}=\max_i s_{t,i}$ denote the maximum entry of the scaling vector $s_t$, then we have 
\[
\|\widetilde{O}_t\|_F^2 =  \sum_{i=1}^{m}s_{t,i}^{2} \le m \hspace{0.3mm}s_{t,\max}^2.
\]
For the \(m>n\) regime, the adaptive update is given by \(\widetilde{O}_t=O_t\operatorname{Diag}(s_t)\). By definition of the polar factor, we have \(O_t^\top O_t=I_n\). Applying an analogous derivation, we obtain
\(\|\widetilde{O}_t\|_F^2
=
\sum_{j=1}^{n}s_{t,j}^2 \le n\hspace{0.3mm} s_{t,\max}^{2}.
\) By introducing \(\ell = \min(m,n)\), both regimes can be unified and compactly written as \(\|\widetilde{O}_t\|_F^2 \le \ell s_{t, \max}^2\). Consequently, the quadratic term in the smoothness inequality satisfies
\[
\frac{\eta^2L}{2}\|\widetilde{O}_t\|_F^2
=
\frac{\eta^2L}{2}
\sum_{k=1}^{\ell}s_{t,k}^2  \le \frac{\eta^2 \ell Ls_{t,\max}^{2}}{2}.
\]
We now analyze the Frobenius inner product term \(\langle \nabla f(W_t), \widetilde{O}_t\rangle_F.\) By decomposing \(\nabla f(W_t)\) into \(M_t\) and the residual \(\nabla f(W_t) - M_t\), we obtain
\[
\langle \nabla f(W_t), \widetilde{O}_t \rangle_F
= \langle M_t, \widetilde{O}_t\rangle_F + \langle \nabla f(W_t) - M_t, \widetilde{O}_t \rangle_F.
\]
The first term characterizes the alignment of the adaptively scaled orthogonal direction with the momentum, and the second term captures the momentum tracking error due to the discrepancy between $M_t$ and the current gradient. We handle these two terms separately in the following. To begin, we address the first term, for which we have the following lemma.

\begin{lemma}[Adaptive Polar Alignment Lower Bound]
\label{polar-align}
Let $M_t$, $O_t$, $\widetilde{O}_t$ and $s_t$ be defined as in Algorithm~\ref{alg:A-Muon}. Let $s_{t,\min}=\min_i s_{t,i}$ denote the minimum entry of the scaling vector $s_t$. Then, the following lower bound on the adaptive polar alignment holds
\[
\langle M_t, \widetilde{O}_t \rangle_F \ge s_{t,\min} \| M_t \|_F.
\]
\end{lemma}

\begin{proof}
We first consider the $m \le n$ regime, where the adaptive scaling is applied via left multiplication: $\widetilde{O}_t = \operatorname{Diag}(s_t)\hspace{0.2mm} O_t$. Substituting $M_t = U_t \Sigma_t V_t^\top$ and $\widetilde{O}_t = \operatorname{Diag}(s_t)\hspace{0.2mm}U_t V_t^\top$ into the Frobenius inner product $\langle M_t, \widetilde{O}_t \rangle_F = \operatorname{tr}(M_t^\top \widetilde{O}_t)$, we obtain
\begin{align*}
\langle M_t, \widetilde{O}_t \rangle_F 
&= \operatorname{tr}( (U_t \Sigma_t V_t^\top)^\top (\operatorname{Diag}(s_t)\hspace{0.2mm} U_t V_t^\top) ) = \operatorname{tr}( V_t \Sigma_t U_t^\top \!\operatorname{Diag}(s_t)\hspace{0.2mm} U_t V_t^\top).
\end{align*}
Using the cyclic property of the trace and the fact that $V_t$ has orthonormal columns (i.e., $V_t^\top V_t = I_{m}$), and expanding the trace, we have
\begin{align*}
\langle M_t, \widetilde{O}_t \rangle_F = \operatorname{tr}( \Sigma_t U_t^\top\! \operatorname{Diag}(s_t)\hspace{0.2mm} U_t) = \sum_{j=1}^{m} \sigma_j \, (U_t^\top \!\operatorname{Diag}(s_t)\hspace{0.2mm} U_t)_{jj},
\end{align*}
where $\{\sigma_j\}_{j=1}^{m}$ are the singular values of $M_t$. By definition, the diagonal entries of $U_t^\top\!\operatorname{Diag}(s_t)\hspace{0.2mm} U_t$ are given by
\[
(U_t^\top\!\operatorname{Diag}(s_t)\hspace{0.2mm} U_t)_{jj} = \sum_{i=1}^{m} s_{t,i} \, U_{t,ij}^2,
\]
where $\{s_{t,i}\}_{i=1}^{m}$ are the entries of $s_t$. Recall that \(s_{t,\min}\) denotes the minimum scaling factor of \(s_t\), i.e.,
\(s_{t,\min}=\min_i s_{t,i}.\)
Then, for each \(j\), we have
\[
\begin{aligned}
\left(
U_t^\top\!\operatorname{Diag}(s_t)\hspace{0.2mm}U_t
\right)_{jj}
&=
\sum_{i=1}^{m}s_{t,i}U_{t,ij}^2\ge
\sum_{i=1}^{m}s_{t,\min}U_{t,ij}^2=
s_{t,\min}\sum_{i=1}^{m}U_{t,ij}^2=
s_{t,\min},
\end{aligned}
\]
where the last equality follows from the column orthonormality of \(U_t\), i.e., \(U_t^\top U_t=I_m\). Therefore, we obtain a lower bound for \(\langle M_t,\widetilde O_t\rangle_F\)
\[
\begin{aligned}
\langle M_t,\widetilde O_t\rangle_F
&=
\sum_{j=1}^{m}\sigma_j(U_t^\top\!\operatorname{Diag}(s_t)\hspace{0.2mm}U_t )_{jj}\ge
s_{t,\min}\sum_{j=1}^{m}\sigma_j=
s_{t,\min}\|M_t\|_*,
\end{aligned}
\]
where $\| \cdot \|_*$ denotes the nuclear norm of a matrix. Finally, since the norm inequality \(\|A\|_*\ge\|A\|_F\) holds for any matrix \(A\)~\citep{horn2013matrix}, it follows that
\[
\langle M_t,\widetilde O_t\rangle_F
\ge
s_{t,\min}\|M_t\|_F .
\]
The $m>n$ regime is analogous, where the adaptive scaling is applied via right multiplication as $\widetilde{O}_t = O_t \operatorname{Diag}(s_t)$. Utilizing the column orthonormality of both $U_t$ and $V_t$, and following an analogous argument, we obtain the same bound. Combining the two regimes, we prove the claimed bound in the lemma.
\end{proof}

Now, by decomposing \(M_t\) into the sum of the gradient \(\nabla f(W_t)\) and the momentum tracking error \(M_t-\nabla f(W_t)\), and applying the reverse triangle inequality for the Frobenius norm, we have
\[
\begin{aligned}
\|M_t\|_F
&=
\|\nabla f(W_t)+(M_t-\nabla f(W_t))\|_F\ge
\|\nabla f(W_t)\|_F
-
\|M_t-\nabla f(W_t)\|_F .
\end{aligned}
\]
Combining this inequality with Lemma~\ref{polar-align}, we obtain
\[
\begin{aligned}
\langle M_t,\widetilde O_t\rangle_F
&\ge
s_{t,\min}\|M_t\|_F\ge
s_{t,\min}\|\nabla f(W_t)\|_F
-
s_{t,\min}\|M_t-\nabla f(W_t)\|_F .
\end{aligned}
\]
We now proceed to address the term $\langle \nabla f(W_t) - M_t, \widetilde O_t \rangle_F$. We first consider its negation. By the Cauchy-Schwarz inequality, we have
\[
\langle M_t - \nabla f(W_t), \widetilde O_t \rangle_F
    \le \|M_t - \nabla f(W_t)\|_F\,\! \|\widetilde O_t\|_F.
\]
As established earlier, we have $\|\widetilde{O}_t\|_F^2 \le \ell s_{t, \max}^2$ with $\ell = \min(m,n)$. Consequently, we obtain
\[
\langle M_t - \nabla f(W_t), \widetilde O_t \rangle_F
    \le \sqrt{\ell}\hspace{0.3mm}\|M_t - \nabla f(W_t)\|_F\hspace{0.3mm} s_{t,\max}.
\]
Therefore, by negating the above inequality, we obtain
\[
\langle \nabla f(W_t) - M_t, \widetilde O_t \rangle_F 
    \ge - \sqrt{\ell}\hspace{0.3mm}\|M_t - \nabla f(W_t)\|_F\hspace{0.3mm} s_{t,\max}.
\]
Substituting the above bounds for the alignment and tracking error terms into the smoothness inequality, we obtain the following descent lemma for DGA-Muon.

\begin{lemma}[Per-step Descent Inequality]
\label{descent-lemma}
Let $W_t$, $M_t$ and $s_t$ be defined as in Algorithm~\ref{alg:A-Muon}. Let $s_{t,\min}$ and $s_{t,\max}$ denote the minimum and maximum entries of $s_t$, respectively, and let $\ell=\min(m,n)$. Assume that $f$ is $L$-smooth. Then, the updates of DGA-Muon satisfy the following per-step descent inequality
\[
f(W_{t+1})\le f(W_t)
    - \eta\,\! s_{t,\min}\|\nabla f(W_t)\|_F + 
    \eta\,\!(s_{t,\min} + \sqrt{\ell}s_{t,\max}) \|M_t - \nabla f(W_t)\|_F
    + \frac{\eta^2\ell L s_{t,\max}^2}{2}.
\]
\end{lemma}

Next, we analyze $\|M_t - \nabla f(W_t)\|_F$, which quantifies the deviation between the stochastic momentum and the true gradient. Recall that $M_t = \beta_1 M_{t-1} + (1-\beta_1) G_t$. Subtracting the true gradient \(\nabla f(W_t)\), we have
\[
M_t - \nabla f(W_t) = \beta_1 M_{t-1} + (1-\beta_1) G_t - \nabla f(W_t).
\]
To isolate the previous-step error, we add and subtract \(\beta_1 \nabla f(W_{t-1})\)
\[
M_t - \nabla f(W_t) = \beta_1 (M_{t-1} - \nabla f(W_{t-1})) + (1-\beta_1)(G_t - \nabla f(W_t)) + \beta_1 (\nabla f(W_{t-1}) - \nabla f(W_t)).
\]
Unrolling recursively from \(t\) down to $1$ gives
\[
\begin{aligned}
M_t - \nabla f(W_t) &= \beta_1^{\,t-1} (M_1 - \nabla f(W_1)) 
+ (1-\beta_1) \sum_{s=2}^{t} \beta_1^{\,t-s} (G_s - \nabla f(W_s)) \\
& + \beta_1 \sum_{s=2}^{t} \beta_1^{\,t-s} (\nabla f(W_{s-1}) - \nabla f(W_s)).
\end{aligned}
\]
Taking the Frobenius norm and applying the triangle inequality, we obtain
\[
\begin{aligned}
\|M_t - \nabla f(W_t)\|_F &\le \beta_1^{\,t-1} \|M_1 - \nabla f(W_1)\|_F 
+ \big\|(1-\beta_1) \sum_{s=2}^{t} \beta_1^{\,t-s} (G_s - \nabla f(W_s))\big\|_F \\
& + \beta_1 \sum_{s=2}^{t} \beta_1^{\,t-s} \|\nabla f(W_{s-1}) - \nabla f(W_s)\|_F.
\end{aligned}
\]
The right-hand side naturally decomposes into three contributions: the initialization term, the stochastic noise term, and the gradient drift term. In the following, we analyze each of them separately.

For the initialization term, since \(M_1=(1-\beta_1)G_1\), we decompose
\[
M_1-\nabla f(W_1)
=
(1-\beta_1)(G_1-\nabla f(W_1))
-\beta_1\nabla f(W_1).
\]
Therefore,
\[
\begin{aligned}
\mathbb E\|
\beta_1^{t-1}(M_1-\nabla f(W_1))
\|_F
\le&
\beta_1^{t-1}(1-\beta_1)
\mathbb E\|G_1-\nabla f(W_1)\|_F
+
\beta_1^t\|\nabla f(W_1)\|_F
\end{aligned}
\]
Using the unbiasedness of $G_1$, Jensen's inequality and the gradient bound assumption, we have
\[
\|\nabla f(W_1)\|_F
=
\|\mathbb E[G_1|\mathcal F_0]\|_F
\le
\mathbb E[\|G_1\|_F|\mathcal F_0]
\le G_{\max}.
\]
Then applying Jensen's inequality, \(\mathbb{E}\|X\|_F \le \sqrt{\mathbb{E}\|X\|_F^{2}}\), and using the bounded-variance assumption
\(\mathbb{E}\|G_1-\nabla f(W_1)\|_F^2 \le \sigma^2\), we obtain
\[
\mathbb{E}\big[\|\beta_1^{\,t-1}(M_1-\nabla f(W_1))\|_F\big]
\le \beta_1^{\,t-1}(1-\beta_1)\sigma + \beta_1^{\,t}G_{\max}.
\]
For the stochastic noise term, we first bound its second moment. For notational convenience, we introduce the shorthand
\(\mathcal{N}_t = (1-\beta_1)\sum_{s=2}^{t} \beta_1^{\,t-s}(G_s - \nabla f(W_s)).\) Expanding the squared Frobenius norm, taking the expectation, and using the linearity of expectation, we obtain
\[
\begin{aligned}
\mathbb{E}\|\mathcal{N}_t\|_F^2
&= (1-\beta_1)^2\,\mathbb{E}\,\! \Big\langle \sum_{s=2}^t \beta_1^{\,t-s} (G_s - \nabla f(W_s)),\; 
\sum_{r=2}^t \beta_1^{\,t-r} (G_r - \nabla f(W_r)) \Big\rangle \\
&= (1-\beta_1)^2 \sum_{s=2}^t \sum_{r=2}^t \beta_1^{\,2t - s - r} \,
\mathbb{E}\,\! \big\langle G_s - \nabla f(W_s),\; G_r - \nabla f(W_r) \big\rangle.
\end{aligned}
\]
By assumption (A2), the stochastic gradient noise \(\xi_t = G_t - \nabla f(W_t)\) forms a martingale difference sequence, i.e., \(\mathbb{E}[\xi_t \mid \mathcal{F}_{t-1}] = 0\), where \(\mathcal{F}_{t-1}\) denotes the filtration representing all information up to time \(t-1\). Hence, the cross terms (\(s \neq r\)) vanish in expectation, leaving only the diagonal terms
\[
\mathbb{E}\|\mathcal{N}_t\|_F^2
= (1-\beta_1)^2 \sum_{s=2}^t \beta_1^{\,2(t-s)} \, \mathbb{E}\|G_s - \nabla f(W_s)\|_F^2
\le (1-\beta_1)^2 \sigma^2 \sum_{k=0}^{t-2} \beta_1^{2k}.
\]
Evaluating the geometric sum gives
\[
\mathbb{E}\|\mathcal{N}_t\|_F^2 
\le (1-\beta_1)^2 \sigma^2 \frac{1 - \beta_1^{2(t-1)}}{1 - \beta_1^2} 
= \sigma^2 \frac{(1-\beta_1)(1 - \beta_1^{2(t-1)})}{1 + \beta_1} 
\le \sigma^2 \frac{1-\beta_1}{1 + \beta_1}.
\]
Then applying Jensen's inequality and using the bounded-variance assumption, we obtain
\[
\mathbb{E}\|\mathcal{N}_t\|_F \le \sqrt{\mathbb{E}\|\mathcal{N}_t\|_F^2} \le  \frac{\sqrt{1-\beta_1}}{\sqrt{1+\beta_1}}\,\!\sigma.
\]
We now proceed to handle the gradient drift term, which, unlike the stochastic noise term, can be bounded deterministically. 
By the $L$-smoothness of $f$, the update rule $W_s=W_{s-1}-\eta \widetilde O_{s-1}$, and the bound  $\|\widetilde{O}_t\|_F \le \sqrt{\ell} s_{t, \max}$, we have
\begin{align*}
\beta_1 \sum_{s=2}^{t} \beta_1^{\,t-s} \|\nabla f(W_{s-1}) - \nabla f(W_s)\|_F
&\le \beta_1 \sum_{s=2}^{t} \beta_1^{\,t-s} L \|W_{s-1} - W_s\|_F \\
&= \beta_1 \eta L \sum_{s=2}^{t} \beta_1^{\,t-s} \|\widetilde{O}_{s-1}\|_F \le \beta_1 \eta \sqrt{\ell} L\sum_{s=2}^{t} \beta_1^{\,t-s} s_{s-1, \max}.
\end{align*}
Let $s_{\max}$ be the maximum of all $s_{t, \max}$, then we obtain
\[
\beta_1 \sum_{s=2}^{t} \beta_1^{\,t-s} \|\nabla f(W_{s-1}) - \nabla f(W_s)\|_F
\le \beta_1 \eta \sqrt{\ell} L s_{\max}\sum_{s=2}^{t} \beta_1^{\,t-s}
\]
For the geometric sum, we have
\[
\sum_{s=2}^{t} \beta_1^{\,t-s} = \sum_{k=0}^{t-2} \beta_1^k = \frac{1 - \beta_1^{t-1}}{1 - \beta_1} \le \frac{1}{1 - \beta_1}.
\]
Therefore, the gradient drift term admits the following deterministic bound
\[
\beta_1 \sum_{s=2}^{t} \beta_1^{\,t-s} \|\nabla f(W_{s-1}) - \nabla f(W_s)\|_F 
\le \frac{\beta_1 \eta \sqrt{\ell} L s_{\max}}{1 - \beta_1}.
\]
Combining the bounds for these three terms yields the following bound on the expected momentum error.

\begin{lemma}[Expected Momentum Error Bound]
\label{momentum-error}
Let $M_t$ and $W_t$ be defined as in Algorithm~\ref{alg:A-Muon}. Let \(s_{\max}\) denote a uniform upper bound on the adaptive scaling factors satisfying \((s_t)_i \le s_{\max} < \infty\) for all $t$ and $i$, and let $\ell=\min(m,n)$. Assume that $f$ is $L$-smooth. Then, under the bounded-variance and bounded-gradient assumptions, the expected momentum error at each step satisfies
\[
\mathbb{E} \|M_t - \nabla f(W_t)\|_F \le \beta_1^{\,t-1}(1-\beta_1)\sigma + \beta_1^{\,t}G_{\max} + \frac{\sqrt{1-\beta_1}}{\sqrt{1+\beta_1}}\sigma + \frac{\beta_1 \eta \sqrt{\ell} L s_{\max}}{1 - \beta_1}.
\]
\end{lemma}

Recall that Lemma~\ref{descent-lemma} establishes the following one-step
descent inequality for the updates of Algorithm~\ref{alg:A-Muon}
\[
\begin{aligned}
f(W_{t+1})
&\le f(W_t)
    - \eta s_{t,\min}\|\nabla f(W_t)\|_F 
    + \eta(s_{t,\min} + \sqrt{\ell}s_{t,\max}) \|M_t - \nabla f(W_t)\|_F
    + \frac{\eta^2\ell L s_{t,\max}^2}{2}.
\end{aligned}
\]
To simplify the subsequent analysis, we replace the time-dependent quantities by uniform bounds over the optimization horizon. Let \(s_{\min}\) denote a uniform lower bound on the adaptive scaling factors satisfying \(0 < s_{\min} \le (s_t)_i\) for all $t$ and $i$. Since
\[
s_{\min}\le s_{t,\min}\le s_{t,\max}\le s_{\max},
\]
we have
\[
-\eta s_{t,\min}\|\nabla f(W_t)\|_F
\le
-\eta s_{\min}\|\nabla f(W_t)\|_F,
\]
and
\[
s_{t,\min}+\sqrt{\ell}\,s_{t,\max}
\le
(1+\sqrt{\ell})s_{\max}.
\]
Moreover, \(s_{t,\max}^2\le s_{\max}^2.\) Therefore, the one-step descent inequality can be uniformly bounded as
\begin{align*}
f(W_{t+1})
&\le f(W_t)
    - \eta s_{\min}\|\nabla f(W_t)\|_F
    + \eta(1+\sqrt{\ell})s_{\max}
      \|M_t-\nabla f(W_t)\|_F
    + \frac{\eta^2\ell Ls_{\max}^2}{2}.
\end{align*}
Rearranging to isolate the gradient term and taking expectations on both sides yields
\begin{align*}
\eta s_{\min} \mathbb{E}\big[\|\nabla f(W_t)\|_F\big]
&\le \mathbb{E}[f(W_t) - f(W_{t+1})] 
+ \eta (1 + \sqrt{\ell})s_{\max} \mathbb{E}\big[\|M_t - \nabla f(W_t)\|_F\big] \\
&+ \frac{\eta^2\ell L s_{\max}^2}{2}.
\end{align*}
Then, summing over $t=1,\dots,T$, using the telescoping sum 
$\sum_{t=1}^{T} \mathbb{E}[f(W_t) - f(W_{t+1})] = \mathbb{E}[f(W_1) - f(W_{T+1})]$ , and subsequently dividing by $\eta s_{\min}$, we obtain
\begin{align*}
\sum_{t=1}^{T} \mathbb{E}\!\left[\|\nabla f(W_t)\|_F\right]
&\le \frac{\mathbb{E}[f(W_1) - f(W_{T+1})]}{\eta s_{\min}} + \frac{(1 + \sqrt{\ell})s_{\max}}{s_{\min}} \sum_{t=1}^{T} \mathbb{E}\!\left[\|M_t - \nabla f(W_t)\|_F\right]\\
&+ \frac{\eta \ell L s_{\max}^2 T}{2 s_{\min}}. 
\end{align*}
Summing the momentum error inequality given by Lemma~\ref{momentum-error} from $t=1$ to $T$ yields
\[
\sum_{t=1}^{T}\mathbb{E}\|M_t-\nabla f(W_t)\|_F
\le (1-\beta_1)\sigma\sum_{t=1}^{T}\beta_1^{\,t-1}
    + G_{\max}\sum_{t=1}^{T}\beta_1^{\,t}
    + T\frac{\sqrt{1-\beta_1}}{\sqrt{1+\beta_1}}\,\sigma
    + T\frac{\beta_1 \eta \sqrt{\ell} L s_{\max}}{1 - \beta_1}.
\]
Evaluating the geometric sums $\sum_{t=1}^{T} \beta_1^{\,t-1} \le \frac{1}{1-\beta_1}$ and $\sum_{t=1}^{T} \beta_1^{\,t} \le \frac{\beta_1}{1-\beta_1}$, we obtain
\[
\sum_{t=1}^{T}\mathbb{E}\|M_t-\nabla f(W_t)\|_F
\le \sigma
    + \frac{\beta_1 G_{\max}}{1-\beta_1}
    + T\frac{\sqrt{1-\beta_1}}{\sqrt{1+\beta_1}}\,\sigma
    + T\frac{\beta_1 \eta \sqrt{\ell} L s_{\max}}{1 - \beta_1}.
\]
Denote \(\Delta = f(W_1)-f^\star\), where \(f^\star = \inf_W f(W)\). Since \(f(W_{T+1})\ge f^\star\), we have
\[\mathbb{E}[f(W_1)-f(W_{T+1})]\le \Delta.\]
Substituting the bound for \(\sum_{t=1}^{T}\mathbb{E}\|M_t-\nabla f(W_t)\|_F\) and \(\mathbb{E}[f(W_1)-f(W_{T+1})] \le \Delta\) into the previous inequality, we obtain the final bound for the cumulative expected gradient norm:
\begin{align*}
\sum_{t=1}^{T} \mathbb{E}\!\left[\|\nabla f(W_t)\|_F\right]
&\le \frac{\Delta}{\eta s_{\min}} + \frac{(1 + \sqrt{\ell})s_{\max}}{s_{\min}} \left( \sigma + \frac{\beta_1 G_{\max}}{1-\beta_1} \right) \\
&\quad + T \frac{(1 + \sqrt{\ell})s_{\max}}{s_{\min}} \frac{\sqrt{1-\beta_1}}{\sqrt{1+\beta_1}}\,\sigma \\
&\quad + T \frac{\beta_1 \eta \sqrt{\ell} L (1+\sqrt{\ell}) s_{\max}^2}{(1 - \beta_1) s_{\min}} + T \frac{\eta \ell L s_{\max}^2}{2 s_{\min}}.
\end{align*}
Dividing both sides by $T$ immediately yields a bound on the average expected gradient norm. Since the minimum is always upper bounded by the average, i.e., $\min_{1 \le t \le T} \mathbb{E}\!\left[\|\nabla f(W_t)\|_F\right] \le \frac{1}{T} \sum_{t=1}^{T} \mathbb{E}\!\left[\|\nabla f(W_t)\|_F\right]$, we obtain the final convergence bound.
\begin{theorem}[Convergence Guarantee of DGA-Muon]
\label{convergence}
Let $W_t$ be defined as in Algorithm~\ref{alg:A-Muon}. Let \(s_{\min}\) and \(s_{\max}\) denote uniform lower and upper bounds on the adaptive scaling factors satisfying \(0 < s_{\min} \le (s_t)_i \le s_{\max} < \infty\) for all $t$ and $i$, and let $\ell = \min(m,n)$. Denote \(\Delta=f(W_1)-f^\star\). Assume that $f$ is $L$-smooth and that the stochastic gradients are unbiased with bounded variance and bounded norm. Then, for a fixed step size $\eta$, DGA-Muon satisfies the following convergence guarantee
\begin{align*}
\min_{1 \le t \le T} \mathbb{E}\!\left[\|\nabla f(W_t)\|_F\right]
&\le \frac{\Delta}{\eta s_{\min} T} + \frac{(1 + \sqrt{\ell})s_{\max}}{s_{\min} T} \left( \sigma + \frac{\beta_1 G_{\max}}{1-\beta_1} \right) \\
&\quad + \frac{(1 + \sqrt{\ell})s_{\max}}{s_{\min}} \frac{\sqrt{1-\beta_1}}{\sqrt{1+\beta_1}}\,\sigma \\
&\quad + \frac{\beta_1 \eta \sqrt{\ell} L (1+\sqrt{\ell}) s_{\max}^2}{(1 - \beta_1) s_{\min}} + \frac{\eta \ell L s_{\max}^2}{2 s_{\min}}.
\end{align*}
In particular, this bound can be summarized as
\[
\min_{1\le t\le T}
\mathbb E\!\left[
\|\nabla f(W_t)\|_F
\right]
=
O\!\left(
\frac1T+\sigma+\eta
\right).
\]
\end{theorem}
Therefore, for a fixed step size \(\eta\), DGA-Muon achieves a convergence rate of \(\mathcal{O}(1/T + \sigma + \eta)\). The adaptive scaling primarily influences the constant factors within the convergence bound. This result reveals the standard bias-variance tradeoff inherent in constant step-size stochastic optimization: the transient optimization error decays at an \(\mathcal{O}(1/T)\) rate, while the stochastic gradient variance and the finite learning rate introduce an irreducible steady-state error. Consequently, DGA-Muon converges to a neighborhood of a stationary point, characterized by a bounded expected gradient norm whose radius depends on the stochastic variance \(\sigma\) and the step size \(\eta\). Similar to classical stochastic optimization methods, a diminishing step-size schedule is expected to gradually eliminate this steady-state error, although extending the present analysis to this setting is beyond the scope of this work. Furthermore, the current analysis establishes a foundational convergence guarantee for DGA-Muon. A rigorous theoretical characterization of when the geometry-aligned adaptive scaling yields provable improvements over the original Muon update remains an open question and is left for future work.

\section{Experiments}

In this section, we report our experimental results. Note that, due to limited computational resources, some experiments were conducted only on a small scale, especially the ablation studies, and are not included in the current version. We will continue to expand the experimental evaluation as additional results become available.

\subsection{Experimental Setup}

We conduct our experiments using NanoChat~\citep{nanochat}, a highly optimized transformer pretraining framework that can serve as a testbed for evaluating optimization methods. Both the model architecture and training recipe have undergone extensive engineering and empirical tuning in the development of NanoChat, including multiple rounds of automated optimization with Autoresearch~\citep{karpathy2026autoresearch}. As a result, NanoChat provides a strong and carefully tuned baseline for optimizer evaluation. We use the default NanoChat model architecture without modification and refer readers to the repository for full implementation details. For our experiments, we use a model that has $24$ layers and $1.38$B parameters in total, with full attention used across all layers. All experiments are conducted on a single node with eight NVIDIA A100 (80GB) GPUs.

We next describe the optimizer configuration used in NanoChat, which incorporates a range of advanced and carefully tuned optimization techniques. NanoChat employs the NorMuon~\citep{li2026normuon} optimizer. As in conventional Muon-based training, the optimizer is applied only to matrix-valued parameters, including the Q/K/V/O and FFN weight matrices, while the remaining parameters are optimized with AdamW~\citep{loshchilov2019decoupled}. For orthogonalization, NanoChat uses the Polar Express~\citep{amsel2026polarexpress} technique. For the matrix-valued parameters, NanoChat employs the recently proposed cautious weight decay~\citep{chen2026cautious} in place of conventional decoupled weight decay. We also ablate this design choice by replacing cautious weight decay with standard decoupled weight decay and find a clear degradation in training performance.

\subsection{Empirical Verification of NorMuon's Scaling Degeneration}

Recall that in Section~\ref{limitation}, we theoretically showed that NorMuon's adaptive scaling degenerates to an approximately constant factor on square matrices. In this experiment, we empirically verify this theoretical characterization by removing the adaptive scaling on the Q/K/V/O matrices, in which case NorMuon effectively reduces to standard Muon. In the current NanoChat implementation, the Q/K/V/O projection matrices are optimized as full matrices over all attention heads rather than separately on a per-head basis, and hence are square. We also tested per-head optimization of the Q/K/V matrices, but observed no improvement over joint optimization.

\begin{figure}
    \centering
    \includegraphics[width=0.75\linewidth]{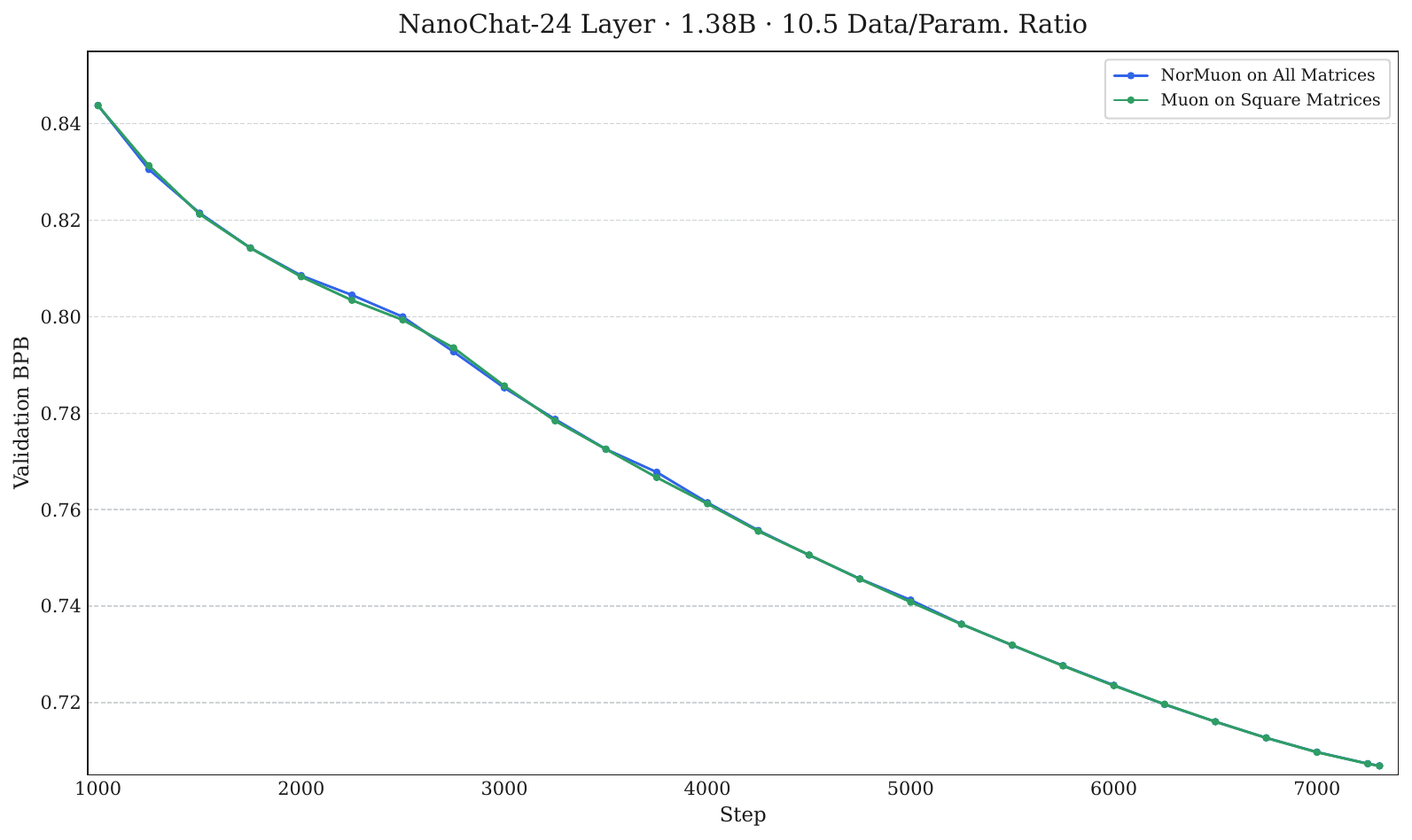}
    \caption{Empirical verification of NorMuon's adaptive scaling degeneration on square matrices in NanoChat. Removing NorMuon's adaptive scaling on the square Q/K/V/O projection matrices, effectively reducing NorMuon to standard Muon, leaves the validation BPB essentially unchanged.}
    \label{fig1}
\end{figure}

Following the NanoChat setup, the total number of training steps is $7,308$, and validation is performed every $250$ steps, with an additional validation run at the end of training. We use validation bits per byte (BPB) as the primary evaluation metric, which measures the average number of bits required to encode each byte of the validation data and is invariant to the tokenizer vocabulary size. Lower BPB indicates better language-modeling performance.

As shown in Figure~\ref{fig1}, removing NorMuon's adaptive scaling on the square Q/K/V/O matrices leaves the validation BPB essentially unchanged, with the curves of the two optimizers nearly overlapping. For the training loss, the two curves are virtually indistinguishable, and we therefore report them in Figure~\ref{fig3} in the Appendix. This result is consistent with our theoretical analysis: for square matrices, the corresponding adaptive scaling factors of NorMuon degenerates to an approximately constant scalar, effectively eliminating its intended adaptive effect. Consequently, removing this scaling has negligible impact on optimization performance, providing empirical support for our theoretical characterization.

\subsection{DGA-Muon Outperforms NorMuon}

We now evaluate the proposed DGA-Muon on NanoChat, using NorMuon as the baseline. We retain the same $24$-layer, $1.38$B-parameter Transformer architecture as in the ablation study above. For this end-to-end evaluation, we increase the data-to-parameter ratio from $10$ to $18$, which is also closer to the Chinchilla-style scaling regime~\citep{hoffmann2022chinchilla}.

Note that DGA-Muon and NorMuon have different optimization characteristics due to their distinct adaptive scaling mechanisms, and therefore require different hyperparameter configurations. We first explain the rationale behind our hyperparameter choices for DGA-Muon. Recall that DGA-Muon constructs its adaptive scaling signal from the raw gradients, whereas NorMuon derives it from the orthogonalized gradients. Orthogonalization provides Muon with favorable isotropic spectral properties, but also removes information encoded in the relative magnitudes of different singular directions. This information becomes increasingly valuable in the later stages of training, as the overall update magnitude decreases and finer-grained differences across singular directions become more important for effective optimization. In contrast, the raw-gradient signal retains this anisotropic information, while exhibiting greater step-to-step variation. These characteristics motivate a more gradual learning-rate schedule for DGA-Muon. Specifically, we extend the warmup phase, increase the warmdown ratio, and reduce the final learning rate fraction, while slightly lowering the learning rate for matrix-valued parameters. The complete hyperparameter configurations are provided in Table~\ref{hyperparameters} in the Appendix.

\begin{figure}
    \centering
    \includegraphics[width=0.75\linewidth]{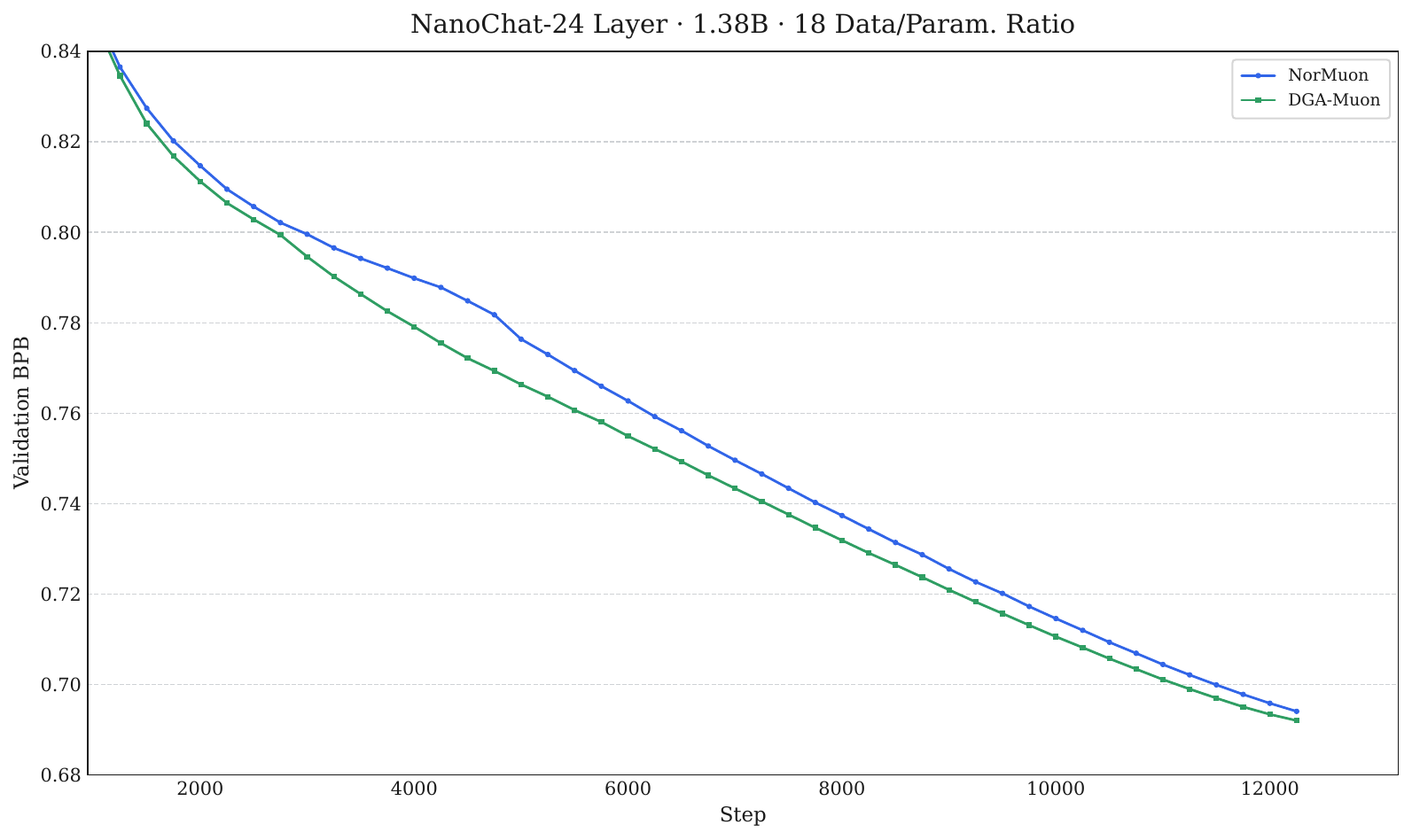}
    \caption{Comparison of DGA-Muon and NorMuon in NanoChat. DGA-Muon consistently outperforms NorMuon, achieving lower validation BPB throughout training.}
    \label{fig2}
\end{figure}

As shown in Figure~\ref{fig2}, DGA-Muon significantly outperforms NorMuon in the end-to-end evaluation, achieving a final validation BPB of $0.6910$, compared with $0.6924$ for NorMuon. The difference in training loss is more pronounced, with DGA-Muon reaching a final loss of $2.2708$ compared with $2.2799$ for NorMuon; the corresponding training-loss curves are provided in Figure~\ref{fig4} in the Appendix. We further evaluate the downstream performance of the two optimizers on the DCLM CORE benchmark, the standard aggregate evaluation metric used in NanoChat. DGA-Muon achieves a CORE score of $0.3089$, while NorMuon obtains $0.2942$, yielding an absolute improvement of $0.0147$ and a relative gain of $5.0\%$. Full downstream results are provided in Table~\ref{eval} in the Appendix. These results validate the effectiveness of our proposed adaptive-scaling design. By constructing adaptive scaling signals directly from raw gradients and aligning the scaling with the underlying geometry of matrices, DGA-Muon achieves more effective adaptation than NorMuon, leading to lower training loss and validation BPB, as well as better downstream performance.

\bibliography{iclr2027_conference}
\bibliographystyle{iclr2027_conference}

\appendix
\section*{Appendix}

\section{Additional experimental Details}
\label{more_exp}

\begin{figure}
    \centering
    \includegraphics[width=0.75\linewidth]{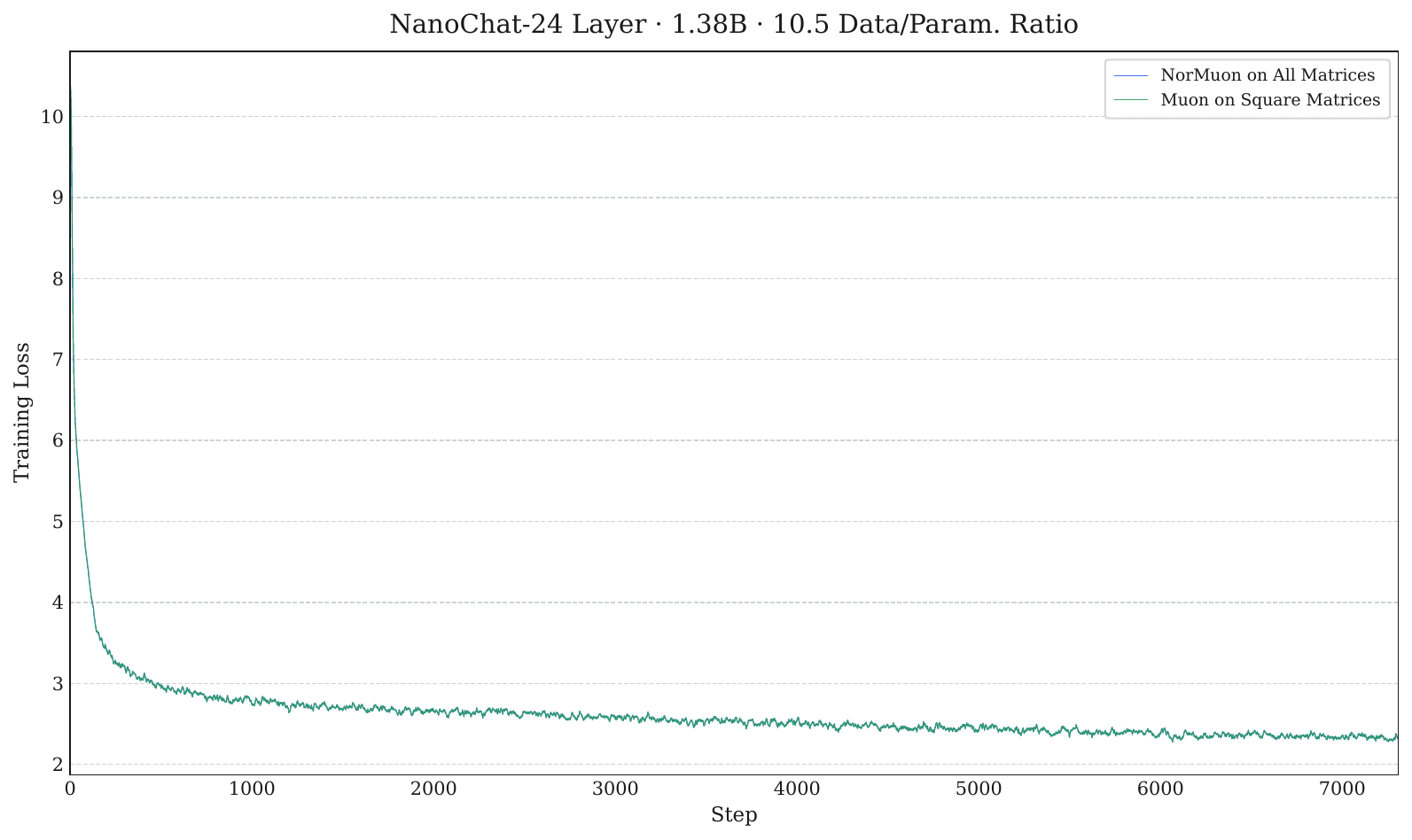}
    \caption{Empirical verification of NorMuon's adaptive scaling degeneration on square matrices in NanoChat. Removing NorMuon's adaptive scaling on the square Q/K/V/O projection matrices, effectively reducing NorMuon to standard Muon, yields virtually indistinguishable training losses.}
    \label{fig3}
\end{figure}

\begin{table}[t]
\centering
\caption{Hyperparameter configurations for the NorMuon baseline and DGA-Muon in NanoChat.}
\vspace{0.5em}
\label{hyperparameters}
\large
\renewcommand{\arraystretch}{1.3}
\setlength{\tabcolsep}{14pt}
\begin{tabular}{|l|c|c|}
\hline
Hyperparameter & NorMuon Baseline & DGA-Muon \\
\hline
Device batch size & 32    & 32 \\
Total batch size  & Auto  & Auto \\
Weight decay      & 0.28  & 0.28 \\
Embedding LR      & 0.3   & 0.3 \\
Unembedding LR    & 0.008 & 0.008 \\
Scalar LR         & 0.5   & 0.5 \\\hline
Matrix LR         & 0.020 & 0.019 \\
Warmup steps      & 40    & 200 \\
Warmdown ratio    & 0.65  & 0.78 \\
Final LR fraction & 0.05  & 0.01 \\\hline
Lower scaling clip & -- & $0.001$ \\
Upper scaling clip & -- & $20$ \\
\hline
\end{tabular}
\end{table}

\begin{figure}
    \centering
    \includegraphics[width=0.75\linewidth]{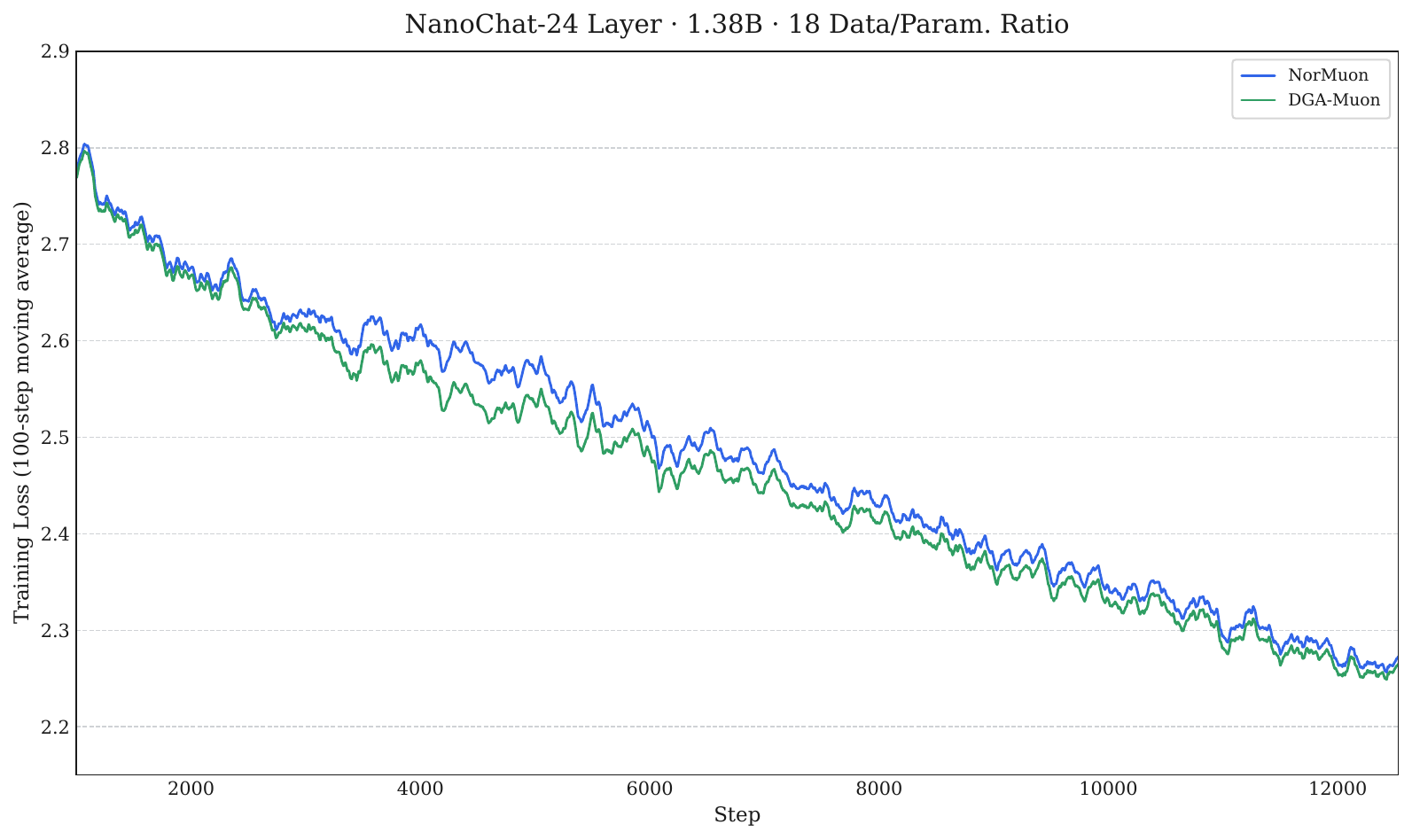}
    \caption{Comparison of DGA-Muon and NorMuon in NanoChat. DGA-Muon consistently outperforms NorMuon, achieving lower training loss throughout training.}
    \label{fig4}
\end{figure}

\begin{table}[t]
\centering
\caption{Downstream task performance of NorMuon and DGA-Muon on the DCLM CORE benchmark. Higher scores indicate better performance.}
\label{eval}
\vspace{0.5em}
\renewcommand{\arraystretch}{1.15}
\setlength{\tabcolsep}{10pt}
\begin{tabular}{|l|c|c|c|}
\hline
Task & NorMuon & DGA-Muon & Diff. \\
\hline
HellaSwag ZeroShot            & 0.6080 & 0.6100 & \textbf{+0.0020} \\
Jeopardy                      & 0.1680 & 0.1840 & \textbf{+0.0160} \\
BIG-Bench QA Wikidata         & 0.5300 & 0.5560 & \textbf{+0.0260} \\
ARC-Easy                      & 0.7460 & 0.7420 & $-$0.0040 \\
ARC-Challenge                 & 0.4760 & 0.4580 & $-$0.0180 \\
COPA                          & 0.6400 & 0.7100 & \textbf{+0.0700} \\
CommonsenseQA                 & 0.2300 & 0.2560 & \textbf{+0.0260} \\
PIQA                          & 0.7800 & 0.7840 & \textbf{+0.0040} \\
OpenBookQA                    & 0.4320 & 0.4080 & $-$0.0240 \\
LAMBADA                       & 0.4880 & 0.4860 & $-$0.0020 \\
HellaSwag                     & 0.6180 & 0.6260 & \textbf{+0.0080} \\
Winograd                      & 0.7143 & 0.6996 & $-$0.0147 \\
Winogrande                    & 0.5760 & 0.5660 & $-$0.0100 \\
BIG-Bench Dyck Languages      & 0.0660 & 0.0940 & \textbf{+0.0280} \\
AGI-Eval LSAT-AR              & 0.2478 & 0.3130 & \textbf{+0.0652} \\
BIG-Bench CS Algorithms       & 0.3940 & 0.4220 & \textbf{+0.0280} \\
BIG-Bench Operators           & 0.2238 & 0.1857 & $-$0.0381 \\
BIG-Bench Repeat Copy Logic   & 0.0312 & 0.0312 & \phantom{+}0.0000 \\
SQuAD                         & 0.5100 & 0.4740 & $-$0.0360 \\
CoQA                          & 0.3220 & 0.3100 & $-$0.0120 \\
BoolQ                         & 0.5660 & 0.6280 & \textbf{+0.0620} \\
BIG-Bench Language Identification & 0.2640 & 0.2500 & $-$0.0140 \\\hline
\textbf{CORE} & \textbf{0.2942} & \textbf{0.3089} & \textbf{+0.0147} \\
\hline
\end{tabular}
\end{table}

\end{document}

%% file: math_commands.tex
\usepackage{amsmath,amsfonts,bm}

\def\eqref#1{equation~\ref{#1}}

\def\1{\bm{1}}

\DeclareMathAlphabet{\mathsfit}{\encodingdefault}{\sfdefault}{m}{sl}
\SetMathAlphabet{\mathsfit}{bold}{\encodingdefault}{\sfdefault}{bx}{n}

